\documentclass[10pt,twocolumn,letterpaper]{article}

\usepackage{cvpr}
\usepackage{times}
\usepackage{epsfig}
\usepackage{graphicx}
\usepackage{amsmath}
\usepackage{amssymb}

\usepackage[pagebackref=true,breaklinks=true,letterpaper=true,colorlinks,bookmarks=false]{hyperref}

\usepackage{caption}
\usepackage{subcaption}
\usepackage{cleveref}
\usepackage{amsthm}

\newcommand\z{\mathbf{z}}
\newcommand\m{\mathbf{m}}
\renewcommand\v{\mathbf{v}}
\newcommand\x{\mathbf{x}}
\newcommand\rv{\mathbf{r}}
\newcommand\n{\mathbf{n}}

\newcommand\s{\mathbf{s}}
\newcommand\te{\mathbf{t}}
\newcommand\ba{\mathbf{b}}

\newcommand\R{\text{R}}
\newcommand\M{\text{M}}
\newcommand\V{\text{V}}
\newcommand\G{\text{G}}
\newcommand\D{\text{D}}
\newcommand\E{\text{E}}
\newcommand\dom{\mathcal{D}}
\newcommand\lgan{{\cal L}_{GAN}}
\newcommand\lae{{\cal L}_{AE}}

\newcommand\lsmooth{{\cal L}_{S}}

\newcommand\mofaspace{0.07}
\newcommand\mofaspaceour{0.14}

\newtheorem{theorem}{Theorem}

\cvprfinalcopy 

\def\cvprPaperID{5991} 

\ifcvprfinal\fi
\begin{document}

\title{Unsupervised 3D shape learning from 2D image collections}
\title{Learning 3D Shapes from Images neither with Templates nor with Annotation}
\title{Unsupervised Learning of 3D Shapes from Image Collections without Templates}
\title{Unsupervised Learning of 3D Shapes from Image Collections}
\title{Unsupervised 3D Shape Learning from 2D Images in the Wild}
\title{Unsupervised 3D Shape Learning from Image Collections in the Wild}

\author{Attila Szab\'{o}\\
University of Bern\\
Switzerland\\
{\tt\small szabo@inf.unibe.ch}
\and
Paolo Favaro\\
University of Bern\\
Switzerland\\
{\tt\small paolo.favaro@inf.unibe.ch}
}

\maketitle

\begin{abstract}
We present a method to learn the 3D surface of objects directly from a collection of images. Previous work achieved this capability by exploiting additional manual annotation, such as object pose, 3D surface templates, temporal continuity of videos, manually selected landmarks, and foreground/background masks. In contrast, our method does not make use of any such annotation. Rather, it builds a generative model, a convolutional neural network, which, given a noise vector sample, outputs the 3D surface and texture of an object and a background image. These 3 components combined with an additional random viewpoint vector are then fed to a differential renderer to produce a view of the sampled object and background. Our general principle is that if the output of the renderer, the generated image, is realistic, then its input, the generated 3D and texture, should also be realistic. To achieve realism, the generative model is trained adversarially against a discriminator that tries to distinguish between the output of the renderer and real images from the given data set. Moreover, our generative model can be paired with an encoder and trained as an autoencoder, to automatically extract the 3D shape, texture and pose of the object in an image. Our trained generative model and encoder show promising results both on real and synthetic data, which demonstrate for the first time that fully unsupervised 3D learning from image collections is possible.
\end{abstract}

\section{Introduction}

%
%
%
%
%
%

In computer vision a fundamental task is to interpret content in images as part of a 3D world. However, this task proves to be extremely challenging, because images are only 2D projections of the 3D world, and thus incomplete observations. 
To compensate for this deficiency, one can use multiple views of the same scene gathered simultaneously, as in multiview stereo \cite{multiview}, or over time, as in structure from motion \cite{sfm}. These approaches are able to combine the information in multiple images through 2D landmark correspondences, because they share the same instance of the world. It is interesting to notice, that no additional information other than the images themselves and a model of image formation are needed to extract 3D information and the pose of the camera. 

In this paper we explore a direct extension of this method to the case where we are given images that never depict the same object twice and aim to reconstruct the unknown 3D surface of the objects as well as their texture. We consider the case where we are made available only a collection of images and no additional annotation or prior knowledge, for example, in the form of a 3D template. Unlike multiview stereo and structure from motion, in this case pixel-based correspondence between different images cannot be established directly.

At the core of our method is the design of a generative model that learns to map images to a 3D surface, a texture and a background image. Then, we use a renderer to generate views given the 3D surface, texture, background image and a randomly sampled viewpoint. We postulate that if the three components are realistic, then also the rendered images from arbitrary (within a suitable distribution) viewpoints will be realistic. To assess realism, we use a discriminator trained in an adversarial fashion \cite{ganGoodfellow}. 
Therefore, by assuming that the samples cover a sufficient distribution of viewpoints, the generative model should automatically learn 3D models from the data.

Finally, we also train an encoder, by pairing it with the generative model and the renderer as an autoencoder. In this training the generative model learns to estimate the 3D surface, the texture, and the viewpoint of an object given an input image. We also formally prove that under suitable assumptions the proposed framework can successfully build such a generative model just from images without additional manual annotation. Finally, we also demonstrate our fully unsupervised method on both real and synthetic data.

\section{Related work}

Traditional 3D reconstruction techniques require multiple views of the objects \cite{sfm, multiview}. They use hand-crafted features \cite{lowe2004distinctive} to match key-points, and they exploit the 3D geometry to estimate their locations. In contrast, 3D reconstruction from a single image is a much more ambiguous problem. Recent methods address this problem by learning the underlying 3D from video sequences \cite{zhou2017kitti}, or by using additional image formation assumptions, such as shape from shading models \cite{horn1989obtaining}. The direct mapping from an image to a depth map can be learned from real data or also from synthetic images as by Sela \etal \cite{sela}. As in structure from motion, Vedaldi \etal \cite{vedaldi2017look} and Zhou \etal \cite{zhou2017kitti} show that it is not necessary to use human annotations to learn the 3D reconstruction of a scene as long as we are given video sequences of the same scene.

In this work we focus on learning a generative model. 3D morphable models (3DMM) \cite{blanz1999, basel2017} are trained with high quality face scans and provide a high quality template for face reconstruction and recognition. Tran \etal \cite{tran2017} and Genova \etal \cite{genova2018} train neural networks for regressing the parameters of 3DMMs. Model-based Face Autoencoders (MoFA) and Genova \etal \cite{mofa2017, genova2018} only use unlabelled training data, but they rely on existing models that used supervision. Therefore, with different object categories, these methods require a new pre-training of the 3DMM and knowledge of what 3D objects they need to reconstruct, while our method applies directly to any category without prior knowledge on what 3D shape the objects have.

Unlike 3DMMs, Generative Adversarial Nets (GAN) \cite{ganGoodfellow} and Variational Autoecoders (VAE) \cite{kingma2014vae} do not provide interpretable parameters, but they are very powerful and can be trained in a fully unsupervised manner. In recent years they improved significantly \cite{arjovsky2017wasserstein, gulrajani2017wgangp, pgan2018}. 3DGANs \cite{wu2016_3dgan} are used to generate 3D objects with 3D supervision. It is possible however to use GANs to train 3D generators by only using 2D images and differentiable renderers similar to the Neural Mesh Renderer \cite{neuralmesh2018} or OpenDR \cite{loper2014opendr}. PrGAN \cite{prgan2017} learns a voxel based representation with GAN, and Henderson \etal \cite{henderson2018} train surface meshes using VAE. They are both limited to synthetic data as they do not model the background. This can be interpreted as using the silhouettes as supervision signal. In contrast we only use 2D image collections and learn a 3D mesh with texture, and model the background as well. PrGAN \cite{prgan2017} as well as our method is a special case of AmbientGAN \cite{bora2018ambientgan}. We extend their theory to the case of 3D reconstruction and describe failure modes, including the Hollow-mask illusion \cite{gregory1970} and the reference ambiguity \cite{szabo2018understanding}.


Our approach can also be interpreted as disentangling the 3D and the viewpoint factors. Reed \etal \cite{reed2015deep} solved that task with full supervision using image triplets. They utilised an autoencoder to reconstruct an image from the mixed latent encodings of other two images. Mathieu \etal \cite{mathieu2016disentangling} and Szab\'{o} \etal \cite{szabo2018understanding} only use image pairs that share an attribute, thus reducing the supervision with the help of GANs. By using only a standard image formation model (projective geometry), by setting a prior on the viewpoint distribution in the dataset, we demonstrate the disentangling of the 3D from the viewpoint and the background for the case of independently sampled input images.


\section{Unsupervised Learning of 3D Shapes}

%
%
%
%
\begin{figure}[t]
\begin{center}
	\begin{subfigure}[b]{0.8\linewidth}
		\includegraphics[width=1\linewidth]{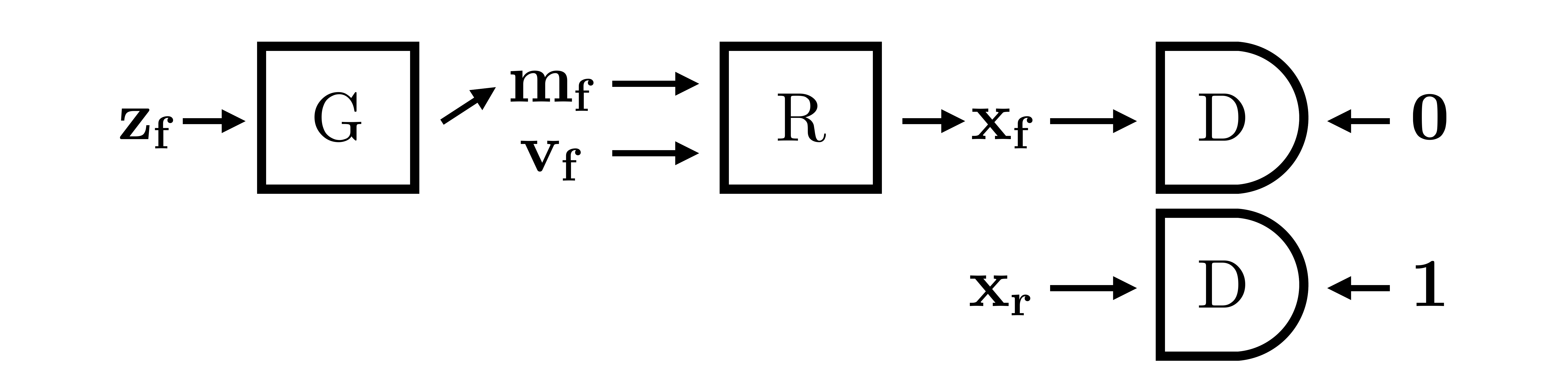}
		\caption{GAN}
	\label{fig:schemaGAN}
	\end{subfigure}
	\begin{subfigure}[b]{0.8\linewidth}
		\includegraphics[width=1\linewidth]{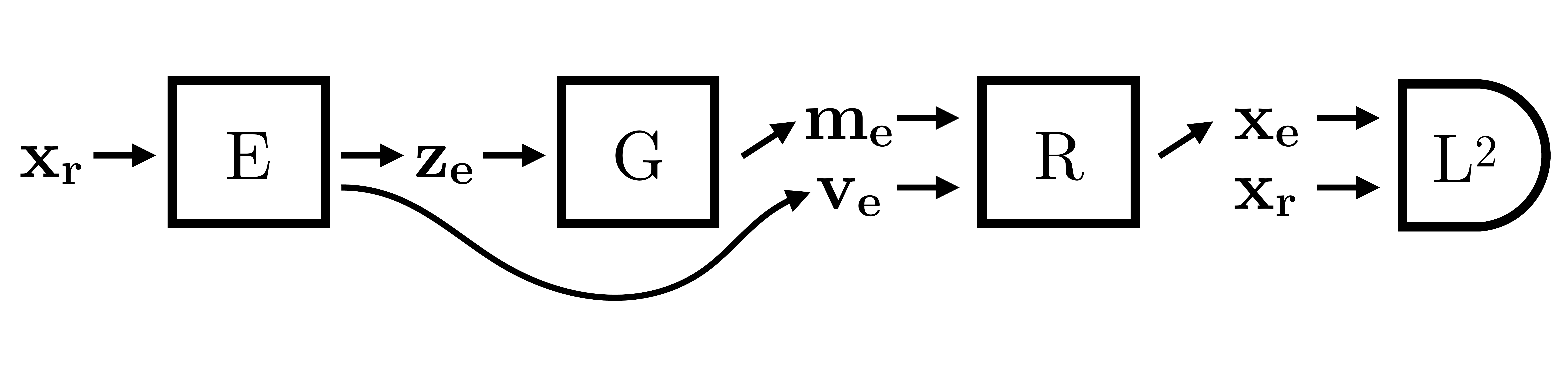}
		\caption{Autoencoder}
			\label{fig:schemaAE}
	\end{subfigure}
	\caption{The GAN and the autoencoder training schemes. First the generator $\G$ and the discriminator $\D$ are trained. Then, the encoder $\E$ is trained with a fixed generator $\G$. The renderer $\R$ has no trainable parameters.}
\end{center}
\end{figure}

We are interested in building a mapping from an image to its 3D model, texture, background image, and  viewpoint. We call all these outputs the \emph{representation}, and, when needed, we distinguish the first three outputs, the \emph{scene representation}, from the viewpoint, the \emph{camera representation}. These outputs are then used by a differential renderer $\R$ to reconstruct an image. Thus, the mapping of images to the representation, followed by the renderer can be seen as an autoencoder. We break the mapping from images to the representation into two steps: First, we encode images to a vector $\mathbf{z_e}$ (with Gaussian distribution) and the camera representation $\mathbf{v_e}$, and second, we decode the vector to the scene representation $\mathbf{m_e}$ (see Figure.~\ref{fig:schemaAE}). The first part is implemented by an encoder $\E$, while the second is implemented via a generator $\G$. In our approach we first train the generator in an adversarial fashion against a discriminator $\D$ by feeding Gaussian samples $\mathbf{z_f}$ as input (see Figure.~\ref{fig:schemaGAN}). The objective of the generator is to learn to map Gaussian samples to scene representations $\mathbf{m_f}$ that result in realistic renderings $\mathbf{x_f}$ for random viewpoints $\mathbf{v_f}$. Therefore, during training we also use random samples for the viewpoints fed to the renderer.
In a second step we then freeze the generator and train the encoder to map images $\mathbf{x_r}$ to vectors $\mathbf{z_e}$ and viewpoints $\mathbf{v_e}$ that allow the generator and the renderer to reconstruct the input $\mathbf{x_r}$.
The encoder can be seen as the general inverse mapping of the generator. However, it is also possible to construct a per-sample inverse of the generator-renderer pair by solving a direct optimization problem.
We implement all the mappings $\E$, $\G$, and $\D$ with neural networks, and we assume they are continuous throughout the paper.
We optimize the GAN objective using WGAN with gradient penalty \cite{gulrajani2017wgangp}.
The chosen representation and all these steps will be described in detail in the next sections.

At the core of our method is the principle that the learned representations will be correct if the rendered images are realistic. This constraint is captured by the adversarial training. We also show analysis that, under suitable assumptions, this constraint is sufficient to recover the correct representation.

\subsection{Enforcing Image Realism through GAN}

In our approach, we train a generator $\G$ to map Gaussian samples $\z_f \sim \mathcal{N}(0,I)$ to 3D models, where $I$ is the identity matrix (and assume that the dimensionality of the samples is sufficient to represent the complexity of the given dataset). We also assume that the viewpoints $\v_f$ are sampled according to a known viewpoint distribution $p_\v$. Then, the GAN loss is
\begin{align}
	\lgan(\G,\D) = E_{\x_r}[ \D(\x_r) ] - E_{\z_f,\v_f}[ \D(\x_f) ],
\end{align}
where $\x_f = \R(\G(\z_f), \v_f )$ are the generated fake images and $\x_r$ are the real data samples. Following Arjovsky \etal \cite{arjovsky2017wasserstein}, we optimize the objective above subject to additional constraints
\begin{align}
	\min_{\G} &\max_{ | \D |_L \le 1} \lgan(\G,\D)\\
	&\text{subject to } |\G(\z_f)|_{3D}\le 1,\quad \z_f \sim \mathcal{N}(0,I),
	\label{eq:gan}
\end{align}
where $ | \cdot |_L $ is the Lipschitz-norm and $|\cdot|_{3D}$ denotes the scale of the 3D model (returned by $\G(\z_f)$). The scale constraint stabilizes the training of the generator especially at the beginning of the training. 


\subsection{Inverting the Generator-Renderer Pair}

In this section we describe a per-sample inversion of the generator and renderer. That is, given a data sample $\x_s$ we would like to find the input vector $\z$ and viewpoint $\v$ that once fed to the trained generator $\G$ and the renderer $\R$ return the image $\x_s$. For simplicity, we restrict our search of the viewpoint to the support of the viewpoint distribution. We denote the support of the distribution $p_\rv$ of a random variable $\rv$ as $\dom_\rv = \overline{ \{ \rv : p_\rv(\rv) > 0\} }$, where $\overline{\{\cdot\}}$ is the closure of a set. We formulate this inverse mapping for a particular data sample $\x_s$ by minimizing the L2 norm
\begin{align}
\begin{split}
	\z^*, \v^* = \arg\min_{\z,\v}  | \x_s - \R(\G(\z), \v) |^2 \\
	\text{subject to} \quad \z \in \dom_\z, \v \in \dom_\v.
\end{split}
\label{eq:invgan}
\end{align}
Finally, the reconstructed 3D model of $\x_s$ is $\m^* = \G(\z^*)$.

\subsection{Extracting Representations from Images}

A more general and efficient way to extract the representation from an image is to train an encoder together with the generator-renderer pair, where the generator has been pre-trained and is fixed. The training enforces an autoencoding constraint so that the input image is mapped to a representation and then through the renderer back to the same image.
In this step, we train only an encoder $\E$ to learn the mapping from an image $\x_r$ to two outputs: $\E_\z(\x_r)$ and $\E_\v(\x_r)$. The first is a vector $\z_e$ and the second a viewpoint $\v_e$ that the generator and the renderer can map back to $\x_r$. The autoencoder loss is therefore
\begin{align}
	\lae(\E) = E_{\x_r}[ | \x_r - \R(\G(\z_e), \v_e) |^2 ] 
\end{align}
where the estimated latent vectors and viewpoints are $\z_e = \E_\z(\x_r)$ and $\v_e = \E_\v(\x_r)$. We denote the estimated 3D model and image as $\m_e = \G(\z_e)$ and $\x_e = \R(\m_e, \v_e)$. Finally, to train the encoder $\E$ we minimize the objective
\begin{align}
\begin{split}
	\min_{\E} & ~~ \lae(\E) \\
	\text{subject to} & ~~ \z_e \in \dom_\z, \v_e \in \dom_\v.
\end{split}
\end{align}

\subsection{3D Representation and Camera Models}

In this section we describe the chosen representation and image formation model. This representation affects the output format $\m$ of the generator $\G$ and also the design of the differential renderer $\R$. We now detail the contents of $\m$ into the object surface $\s$, the object texture $\te$ and the background image $\ba$.
 
\noindent\textbf{3D Surface and Texture. }
Currently, our representation considers a single 3D object and its texture. The object surface $\s$ consists of a mesh of triangles (a fixed number). The mesh is given as a list of vertices and the list of triangles, which is fixed and consists of triplets of vertex indices. The vertices have associated RGB colors and 3D coordinates indicating its position in a global reference frame. Therefore, $\s$ is a vector of 3D coordinates and $\te$ is a vector (of the same size) of RGB values. 
We illustrate graphically this representation in Figure.~\ref{fig:representation3D}. 
\begin{figure}[t]
\centering
		\includegraphics[width=.8\linewidth,trim={2cm 10cm 24cm 6.7cm},clip]{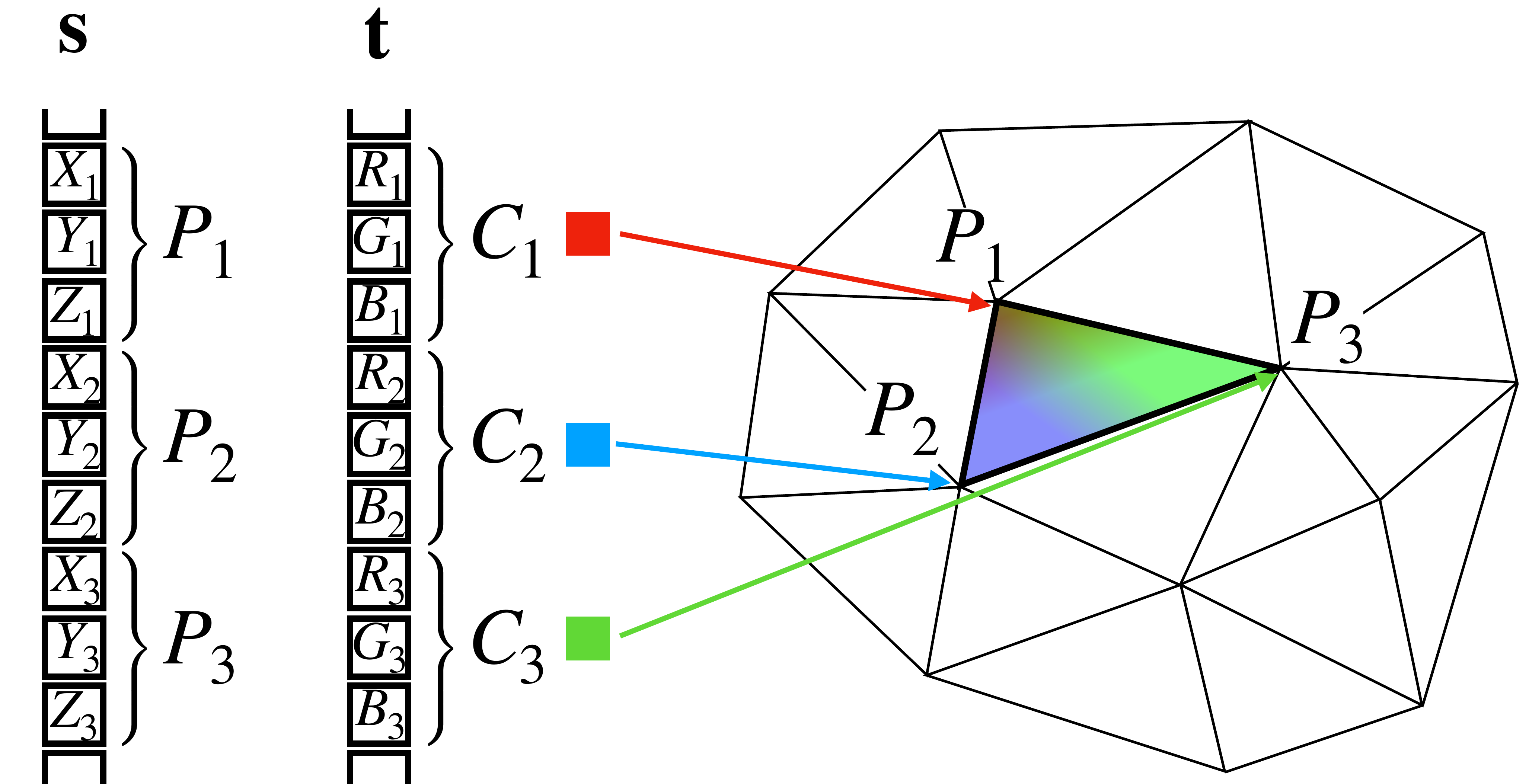}
	\caption{Illustration of the 3D and texture representation.}
	\label{fig:representation3D}
\end{figure}

\noindent\textbf{Background. }
We explicitly model the background $\ba$ behind the object to avoid the need for supervision with silhouettes/masks. Also, we move the background when we change the viewpoint. This is used to make sure that the generator avoids learning trivial representations. For example, if we used a static background, the generator $\G$ could learn to place the object outside the field of view of the camera and to simply map the whole input image to the background texture. 
In our representation we fix the 3D coordinates of the background on a sphere centered around the origin and only learn its texture. To image the background we also approximate the image projection on the plane with a projection on a spherical image plane (a small solid angle), so that the resolution of the background would not change across the image. With this model, viewpoint changes, \ie, rotations about the origin (we only consider 2 angles), correspond to 2D shifts of the background texture $\ba$. Also, to avoid issues with the resolution of the background, we simulate a different scaling of the background rotation compared to the object rotation. This results in a much larger field of view for the background compared to the field of view used for the object.
Finally, by using the assumption that the objects are around the center of the image, it would not be possible for the generator to learn to place the objects on the background, because a change in the viewpoint will move the background and with it the object (thus away from the center of the image). 

\noindent\textbf{Camera Model and Viewpoint. }
Our camera model uses perspective projection. However, to minimize the distortion we use a large focal length (\ie, the distance between the image plane and the camera center). The image plane is also placed at the origin. When we change the viewpoint, we rotate the camera around the origin and thus around its image plane. In this way the image resolution and distortions are quite stable for a wide range of viewpoints. We only consider 2 rotation angles in Euler coordinates along the $x$ and $y$ axis. 


\subsection{Network Architectures}

\noindent\textbf{Generator.}
Based on the chosen representation, our generator consist of three sub-generators: one that learns the background texture $\ba$, a second that learns the object texture $\te$ and a third that learns the object geometry $\s$. 
The latent input vector $\z$ is split into two parts: $\z_b$ for the background image and $\z_o$ for the object, since we generate the background and the object independently, but then expect correlation between the object texture and 3D. The sub-generators for the object texture and the object surface receive as input $\z_o$, while the sub-generator for the background image receives as input $\z_b$. The outputs of all the sub-generators have the same output size since all have 3 channels and the same resolution.
For the object and background texture sub-generators we use two separate models based on the architecture of Karras \etal \cite{pgan2018}.

The object surface is initially represented in spherical coordinates, where the azimuth and polar angles are defined on a uniform tessellation and we only recover the radial distance.
Learning spherical coordinates helps to keep the resolution of the mesh consistent during training, in contrast to learning the 3D coordinates directly.
The radial distance is then obtained as a linear combination of $32$ basis radial distances and an average radial distance. The object surface sub-generator recovers both the coefficients of the linear combination and the $32+1$ basis radial distances. While the basis is represented directly as network parameters, the coefficients are obtained through the application of a fully connected layer to the input $\z_o$.
To ensure a smooth convergence, we use a redundant coarse-to-fine basis.
The radial distance $\rho$ produced by the object surface sub-generator can be written as
\begin{align}
\rho = \sum_{i=0}^{32} \alpha_i \sum_{n=2}^6 \text{upsample}_n(\rho_{i,n})
\end{align}
where $\alpha_i$ is the $i$-th output of the fully connected layer of the sub-generator, $n$ is the level of coarseness of $\rho_{i,n}$, which corresponds to the resolutions $(2^n+1)\times (2^n+1)$, and $\text{upsample}_n$ is a bilinear interpolation function that maps the $n$-th scale to the final resolution of $129\times 129$ elements.
Although the representation of the radial distance $\rho$ is redundant (the resolution $\rho_{i,m}$ can also fully describe $\rho_{i,n}$ with $n<m$), we find that this hierarchical parametrization helps to stabilize the training. 
Finally, the estimated radial distance $\rho$, together with the azimuthal angle $\phi$ and polar angle $\theta$ are mapped to 3D coordinates $\s$ via the spherical coordinates transform and then given as input to the renderer.

\noindent\textbf{Discriminator and Encoder.}
The discriminator architecture is the same as in \cite{pgan2018}. We always rendered the images at full resolution (at $128 \times 128$ pixels), and then downsampled it to match the expected input size of the discriminator during the training with growing resolutions. The encoder architecture is also the same as the discriminators, with small modifications. The output vector size is increased from $1$ to $512$, and two small neural networks were attached to produce the two encoder outputs. The latent vectors $\z$ were estimated with a fully connected layer. The $3$ Euler angles were discretized and their probabilities were estimated by a fully connected layer and a softmax. Then a continuous estimate was computed by taking the expectation.

\begin{figure}[t]
\begin{center}
	\begin{subfigure}[b]{0.22\linewidth}
		\includegraphics[width=1\linewidth]{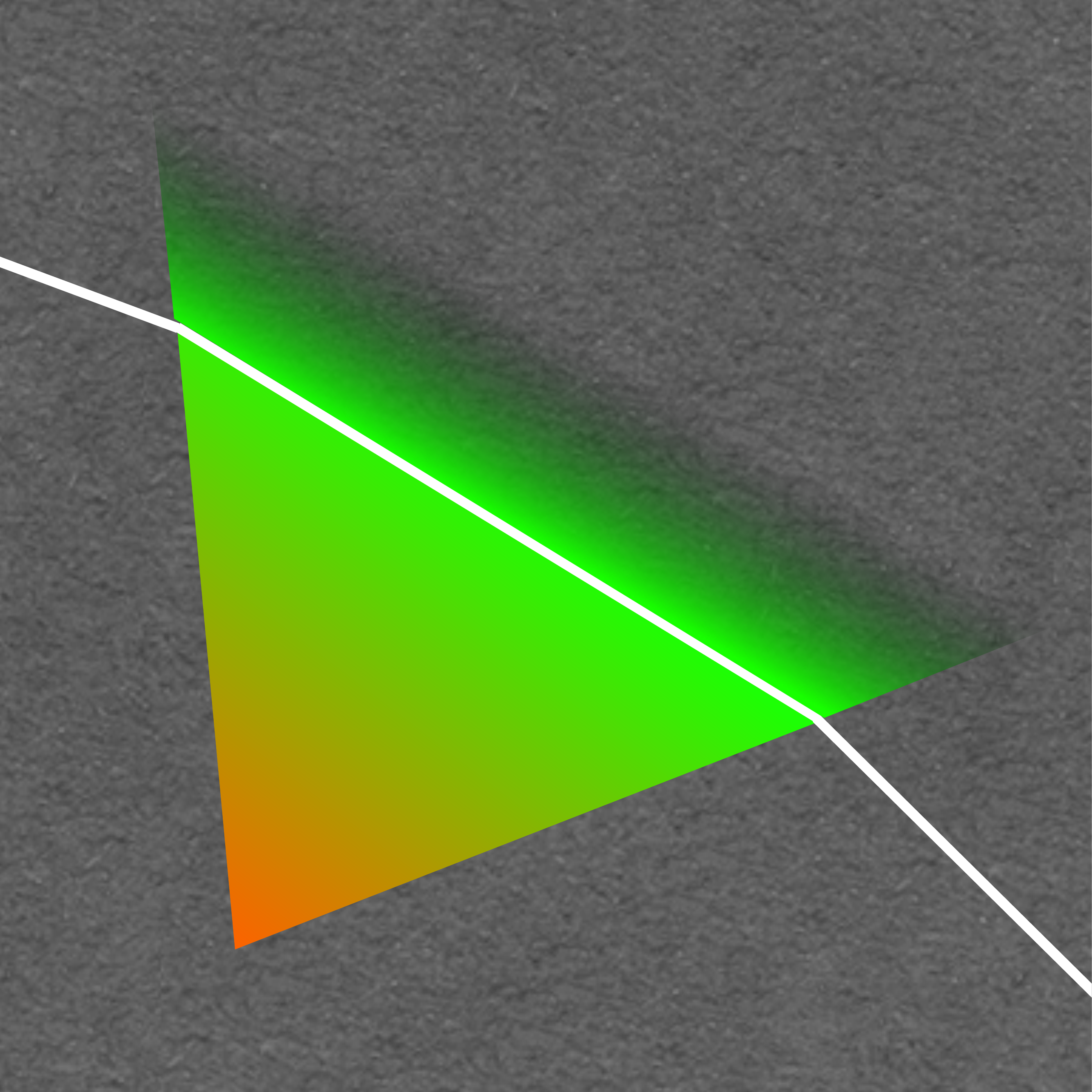}
	\end{subfigure}
	\begin{subfigure}[b]{0.22\linewidth}
		\includegraphics[width=1\linewidth]{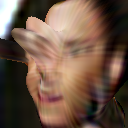}
	\end{subfigure}
	\begin{subfigure}[b]{0.22\linewidth}
		\includegraphics[width=1\linewidth]{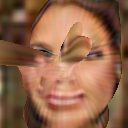}
	\end{subfigure}
	\begin{subfigure}[b]{0.22\linewidth}
		\includegraphics[width=1\linewidth]{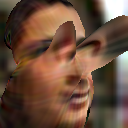}
	\end{subfigure}
	
	\caption{Blurred renderer. The leftmost image illustrates how the triangles are blurred at the boundaries. The other three images show failure cases of the non-blurred renderer. In this case, spikes and sails tend to appear on the 3D shape.}
	\vspace{-5 mm}
	\label{fig:ablation}
\end{center}
\end{figure}

\subsection{Differential Renderer}

The renderer takes as input the representation $\m$ from $\G$, \ie, the surface $\s$, the RGB colors $\te$, and the background image $\ba$, and the viewpoint $\v$ of the camera. For simplicity, we do not model surface and light interaction with shading, reflection, shadows or interreflection models. We simply use a Lambertian model. Since each vertex in $\s$ is associated to a color in $\te$, we interpolate colors inside the triangles using barycentric coordinates. This color model allows the back-propagation of gradients from the renderer output to the vertex coordinates in $\s$ and colors in $\te$. While the differentiation works for pixels inside triangles, at object boundaries and self-occlusions the gradients can not be computed. While others approximate gradients at boundaries \cite{neuralmesh2018, loper2014opendr}, we instead modify our rendering engine to draw blurred triangles, so the gradients can be computed exactly. The blurring process is illustrated in Figure.~\ref{fig:ablation}. At the boundaries and self-occlusions the triangles are extended and matted linearly against the background or the occluded part of the object. The effect of this blurred rendering is visually negligible, and it only effects a few pixels at the boundaries. However, we found that it contributes substantially to the stability of the training. In contrast, the non-blurry renderer often generates spikes and sails, which made the training unstable. Examples of this behavior are shown in Figure.~\ref{fig:ablation}.

\section{A Theory of 3D Generative Learning} \label{theory}

In this section we give a theoretical analysis of our methods. We prove that under reasonable conditions the generator $\G$ can output realistic 3D models. In \Cref{gantheorem} we adapted the theory in AmbientGAN \cite{bora2018ambientgan} to our approach, as AmbientGAN uses vanilla \cite{ganGoodfellow} and we use Wasserstein GAN \cite{arjovsky2017wasserstein, gulrajani2017wgangp}. We can provably obtain the 3D model of the object by inverting the generator $\G$ or by estimating it via the autoencoder. Our theory requires five main assumptions:
\begin{enumerate}
	\item The real images in the dataset are formed by the differentiable rendering engine $\R$ given the independent factors $\m_r \sim p_\m$ and $\v_r \sim p_\v$. The images are formed deterministically by $\x_r = \R(\m_r, \v_r)$. This assumption is needed to guarantee that the generator with the renderer can perfectly model the data. Note that $p_\m$ is assumed unknown; 
	\item We know the viewpoint distribution $p_\v$, so we can sample from it, but we do not know the viewpoint $\v_r$ for any particular data sample $\x_r$;
	\item The rendering engine R is bijective in the restricted domain $\dom_\m \times \dom_\v$, and we denote the inverse with $\R^{-1} = [ \M, \V]$, where $\M(\x_r) = \m_r$ and $\V(\x_r) = \v_r$. This property has to be true for any deterministic viewpoint or 3D estimator, otherwise the (single image) 3D reconstruction task can not be solved. Note that this is assumption is also needed in fully supervised methods;
	\item There is a unique probability distribution $p_\m$ that induces the distribution $p_\x$, when $\v \sim p_\v$. This assumption is not true for 3D data in general, as the 3D objects can have many symmetries. We will discuss and show ambiguities in the experimental section;
	\item Finally, we assume that the encoder, generator and discriminator have infinite capacity, and the training reaches the global optimum $\lgan = 0$ and $\lae=0$. Note that our first and second assumptions are a necessary condition for perfect training.
\end{enumerate}
Now we show that the generator learns the 3D geometry of the scene faitfully.
\begin{theorem} When the generator adversarial training is perfect, \ie, we achieve $\lgan = 0$, the generated scene representation distribution is identical to the real one, thus $G(\z_f) \sim p_\m$, with $\z_f\sim {\cal N}(0,I)$.
\label{gantheorem}
\end{theorem}
\begin{proof}
Since the training is perfect, the discriminator is also perfect. Then, the GAN loss is equal to the Wasserstein distance between $p_\x$ and $q_\x$, where $q_\x$ is the distribution of the generated fake data $\x_f \sim q_\x$. As the distance is zero, $p_\x$ and $q_\x$ are identical. This implies that $\m_f \sim p_\m$, with $\m_f = G(\z_f)$ and $\z_f\sim {\cal N}(0,I)$, as only $p_\m$ can induce the real distribution $p_\x$
\end{proof}
Next, we show that the ideal generator $\G$ obtained in Theorem~\ref{gantheorem} above can be inverted for a particular data sample $\x_s$ by solving~\eqref{eq:invgan}.
\begin{theorem} Suppose that the ideal generator $\G$ in Theorem~\ref{gantheorem} is given. Then, when the objective in Problem~\eqref{eq:invgan} achieves the global minimum, \ie, $\R(\G(\z^\ast), \v^\ast)=\x_s$,
the estimated scene representation $\G(\z^\ast)$ and the viewpoint $\v^\ast$ are correct, \ie, $\G(\z^*) = \M(\x_s)$ and $\v^* = \V(\x_s)$.
\label{invgantheorem}
\end{theorem}
\begin{proof}
Let us denote $\m^* = \G(\z^*)$. $\G$ is continuous and $\z^* \in \dom_\z$, therefore $\m^* \in \dom_\m$ (otherwise $\lgan \ne 0$). Because $\x_s = \R(\m^*,\v^*)$, and R is invertible on $\dom_\m \times \dom_\v$, the inverses are $\M(\x_s) = \m^*$ and $\V(\x_s) = \v^*$.
\end{proof}
Lastly, we show that the encoder $\E$ combined with the ideal generator $\G$ reconstructs the correct surface of the object in the input image.

\begin{figure*}
\begin{center}
	\begin{subfigure}[b]{0.95 \textwidth}
		\includegraphics[width=\textwidth]{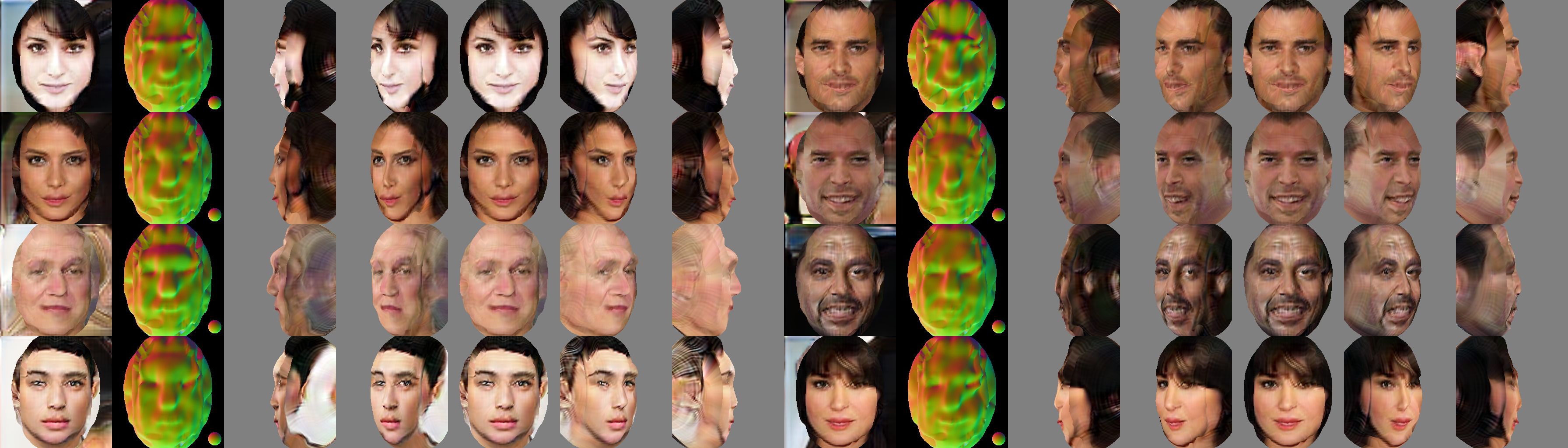}
	\end{subfigure}
\end{center}
	\caption{Samples from our generator on CelebA. The leftmost image has frontal view and background. The second shows the normal map coloured according to the reference sphere in the bottom right. The leftmost five images show views between $\pm 90$ degrees.}
\label{fig:gan}
\end{figure*}

\begin{figure}[t]
\begin{center}	
	\begin{subfigure}[b]{0.95 \linewidth}
		\includegraphics[width=1\linewidth]{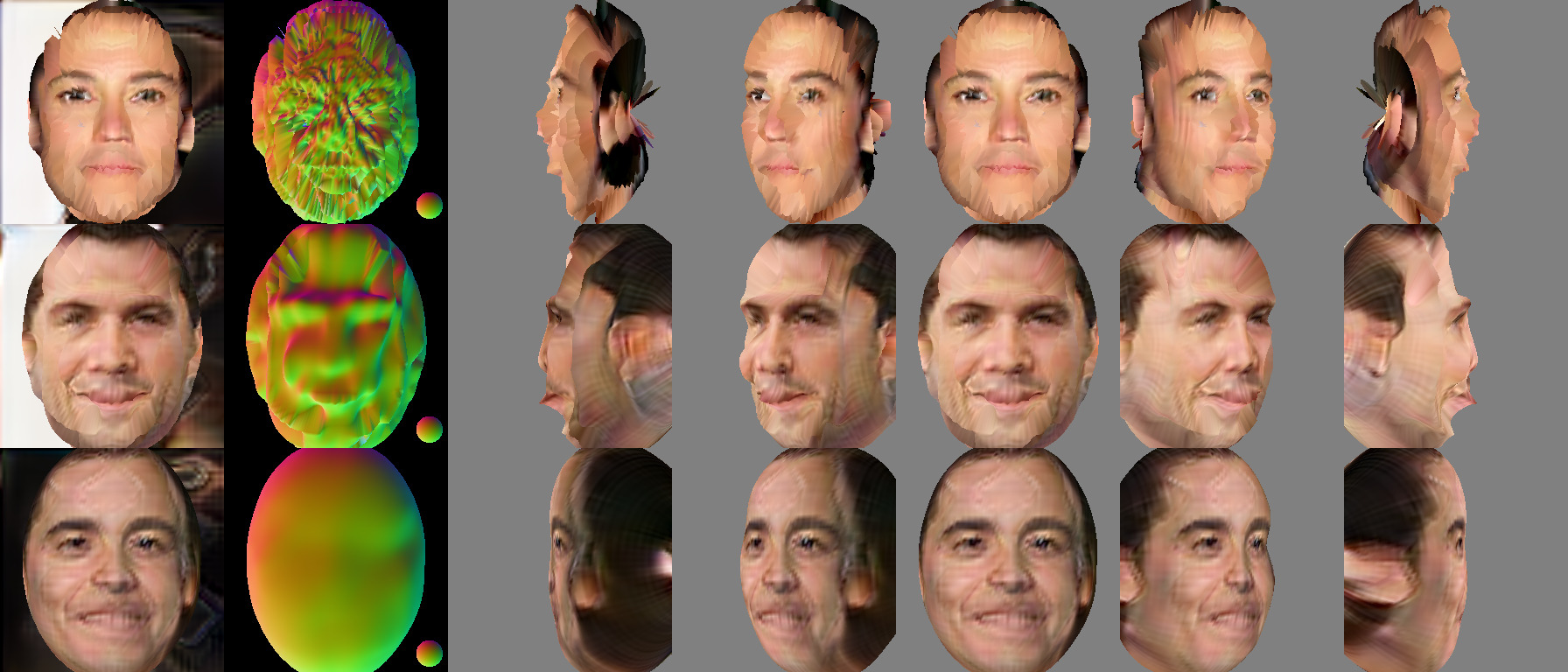}
	\end{subfigure}	
	\caption{Smoothing. From top to bottom the smoothing coefficient $\lambda_S$ is increased, with the top row having $\lambda_S = 0$. The images are rendered as in Figure~\ref{fig:gan}.}
	\label{fig:smooth}
\end{center}
\end{figure}

\begin{figure}[t]
\begin{center}	
	\begin{subfigure}[b]{0.95 \linewidth}
		\includegraphics[width=1\linewidth]{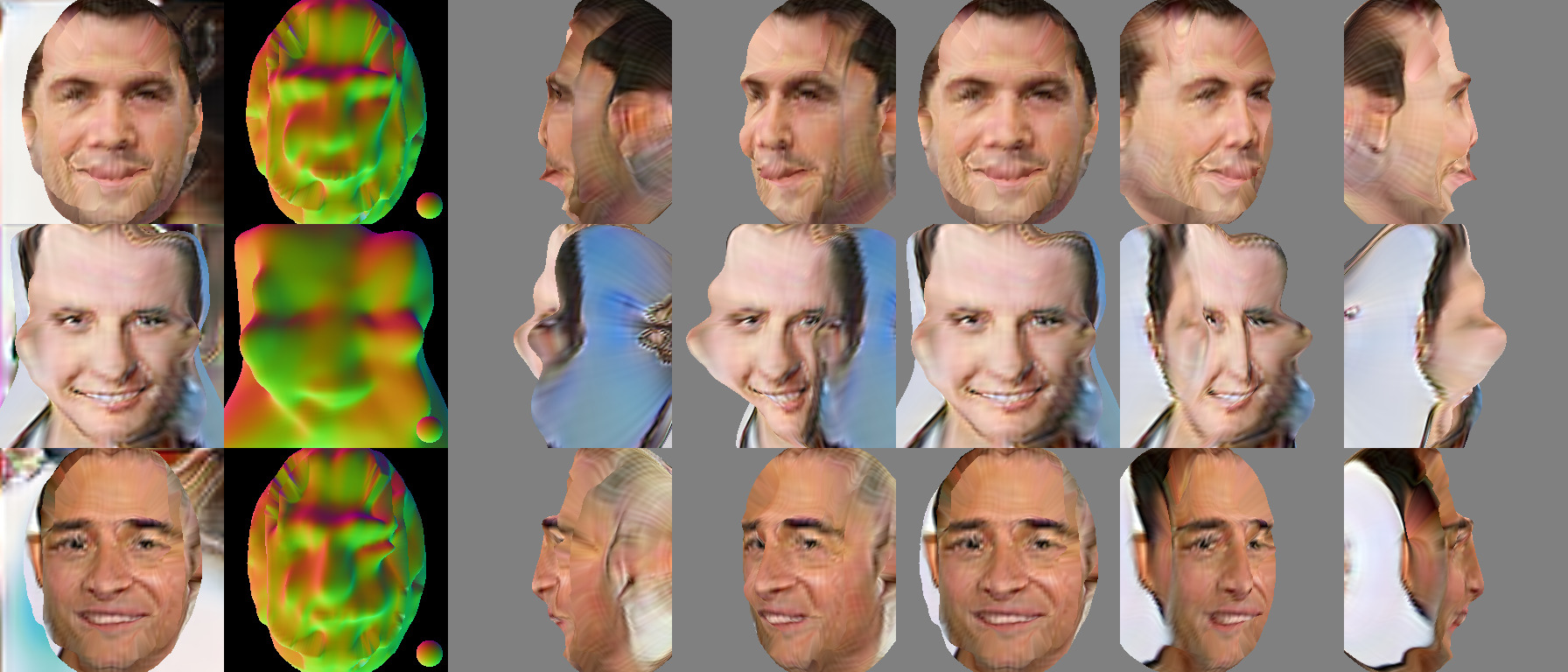}
	\end{subfigure}	
	\caption{Ambiguities. From top to bottom we show a correct sample, one with the hollow-mask, and one with the reference ambiguity.}
	\vspace{-5 mm}
	\label{fig:ambiguities}
\end{center}
\end{figure}

\begin{theorem} When the training is prefect ($\lgan=0$ and $\lae=0$), the estimated scene representation $\m_e = \G(\E_\m(\x_r))$ and viewpoint $\v_e = \E_\v(\x_r)$ are correct, \ie $\m_e = \M(\x_r)$ and $\v_e = \V(\x_r)$.
\label{aetheorem}
\end{theorem}
\begin{proof}
Because $\E$, $\G$ and $\R$ and their compositions are continuous, $\x_r = \x_e$, otherwise  $\lae \ne 0$. The estimated scene representation $\m_e \in \dom_m$, because $\z_e \in \dom_\z$. Finally, the invertibility of $\R$ implies that $\m_e = \M(\x_e) = \M(\x_r)$.
\end{proof}
In general, assumption four does not hold for the 3D reconstruction problem. Depending on the dataset and the symmetries of the objects, ambiguities can arise. One notable failure case is the \emph{hollow-mask illusion} \cite{gregory1970}. An inverted mask (a concave face) can look realistic, even though it is far from the true geometry. This failure mode can be overcome when the range of viewpoints is large enough so the self-occlusions give away the depth information. In our experiments we observe cases where the system learns inverted faces, but this ambiguity never appears on ShapeNet objects, as they are rendered with a large range of viewpoints.
Another example is the \emph{reference ambiguity} \cite{szabo2018understanding}. Two different 3D models can both be realistic, but their reference frames are not necessarily the same, \ie, when they are rendered with the same numerical values of viewpoint angles, they are not aligned.
We show examples of both of these ambiguities in Figure~\ref{fig:ambiguities}.

It is important to note that the above mentioned problems are different from the ill-posedness of the single image 3D reconstruction problem, which means that many different 3D objects can produce the same 2D projection. When the viewpoint distribution is known, we can render a candidate 3D shape from a different viewpoint, which immediately reveals, whether it is realistic or not. Ill-posedness is only a problem if the viewpoint distribution is not known. If one has to estimate the viewpoint distribution as well, a trivial failure mode could emerge. The encoder and generator would learn to map images to a flat surface with a fixed viewpoint and the textures would match exactly the 2D inputs.

\begin{figure*}
\begin{center}
	\begin{subfigure}[b]{0.47 \textwidth}
		\includegraphics[width=\textwidth]{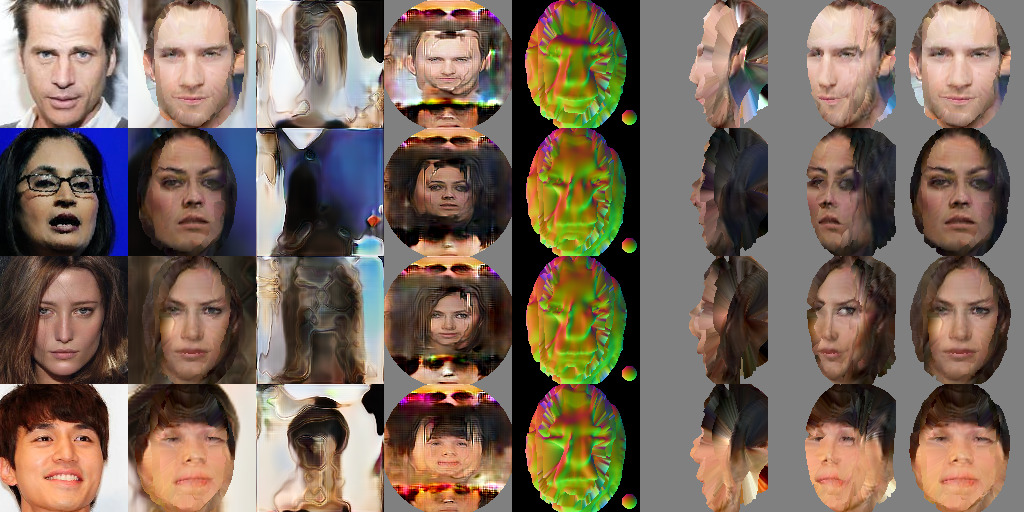}
		\caption{Autoencoder}
	\end{subfigure}
	\begin{subfigure}[b]{0.47 \textwidth}
		\includegraphics[width=\textwidth]{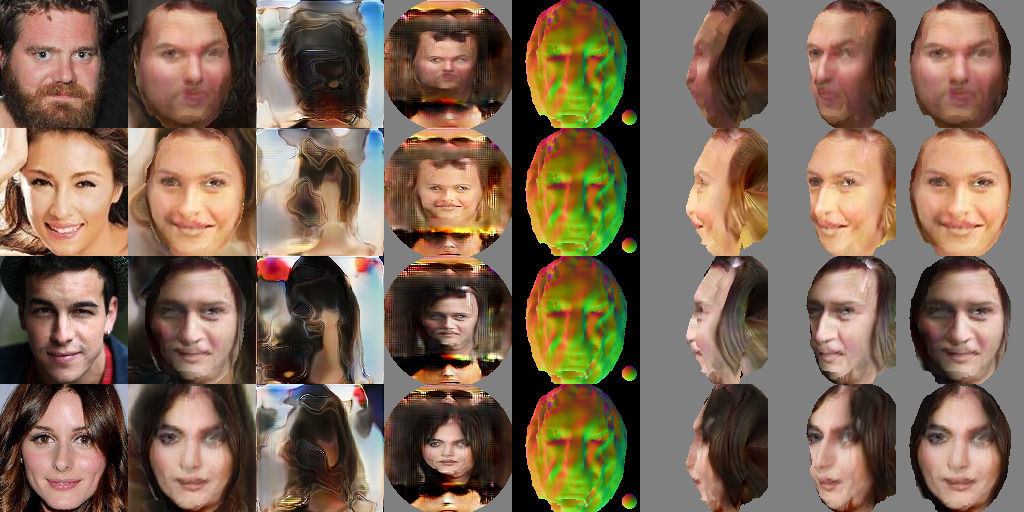}
		\caption{Autoencoder with smoothing}
	\end{subfigure}
\end{center}
	\caption{Reconstructions with two different autoencoders: \textbf{(a)} only autoencoder loss is optimized, \textbf{(b)} autoencoder and smoothing loss optimized. The images show from left to right: input, estimated image, background texture, object texture, object 3D and renderings from $3$ viewpoints. The grey portions of the texture is not used in our generator model.}
\label{fig:autoencoder}
\end{figure*}

\begin{figure*}
\begin{center}	
	\begin{subfigure}[b]{0.95 \linewidth}
		\includegraphics[width=1\linewidth]{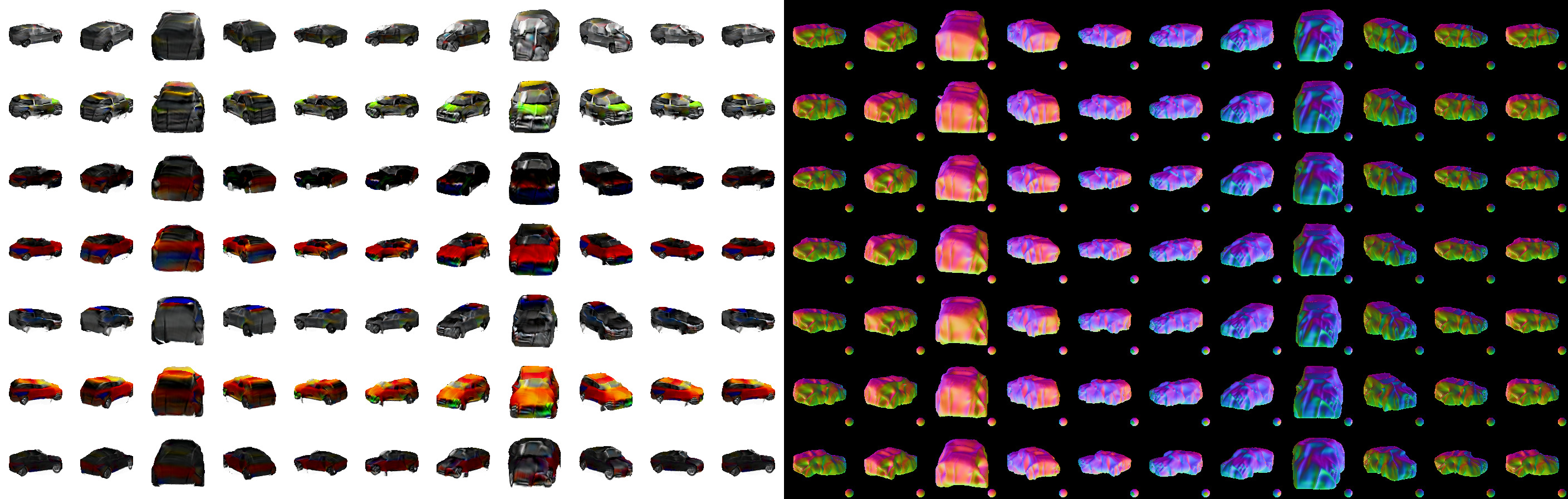}
	\end{subfigure}	
	\caption{Samples from our generator on ShapeNet. Each row shows a rendered car from different viewpoint. The 3D normals are shown on the right side.}
	\vspace{-5 mm}
	\label{fig:shapenet}
\end{center}
\end{figure*}

\begin{figure*}[t]
\begin{center}
	\begin{subfigure}[b]{\mofaspace \textwidth}
		\includegraphics[width=\textwidth]{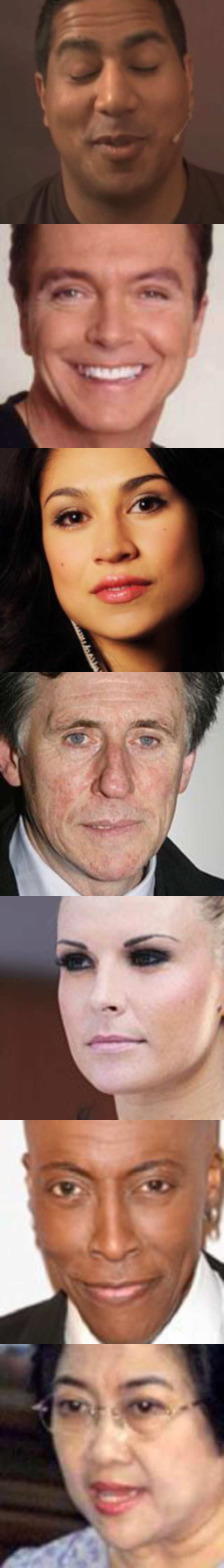}
		\caption{}
	\end{subfigure}
	\quad
	\begin{subfigure}[b]{\mofaspace \textwidth}
		\includegraphics[width=\textwidth]{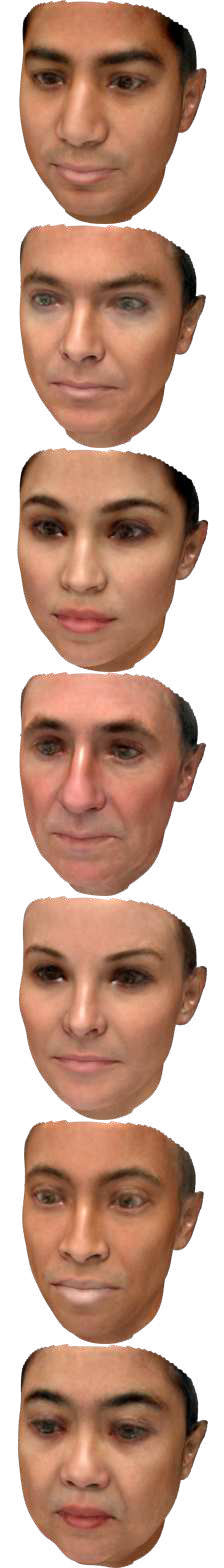}
		\caption{}
	\end{subfigure}
	\quad
	\begin{subfigure}[b]{\mofaspace \textwidth}
		\includegraphics[width=\textwidth]{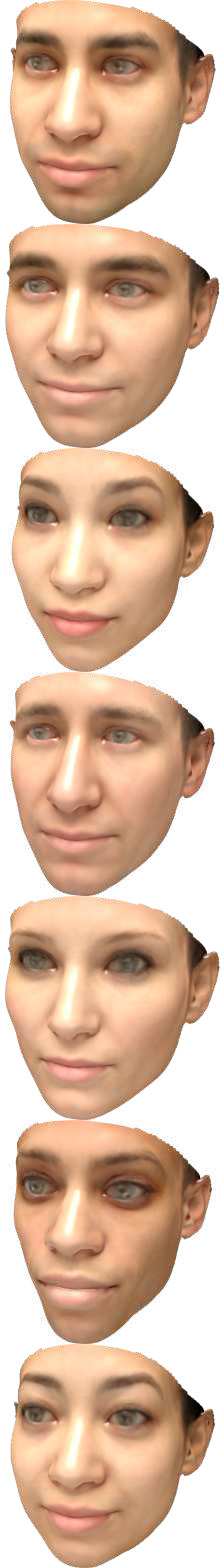}
		\caption{}
	\end{subfigure}
	\quad
	\begin{subfigure}[b]{\mofaspace \textwidth}
		\includegraphics[width=\textwidth]{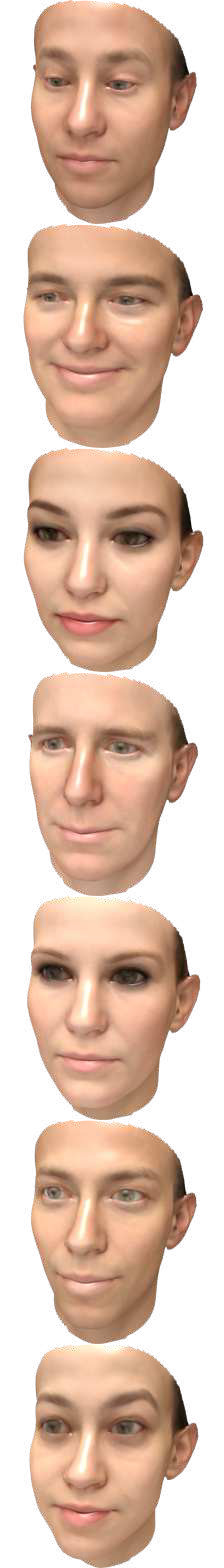}
		\caption{}
	\end{subfigure}
	\quad
	\begin{subfigure}[b]{\mofaspace \textwidth}
		\includegraphics[width=\textwidth]{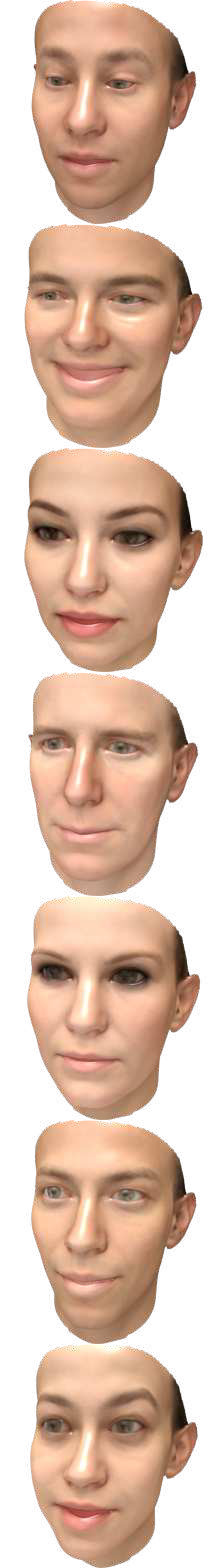}
		\caption{}
	\end{subfigure}
	\quad
	\begin{subfigure}[b]{\mofaspace \textwidth}
		\includegraphics[width=\textwidth]{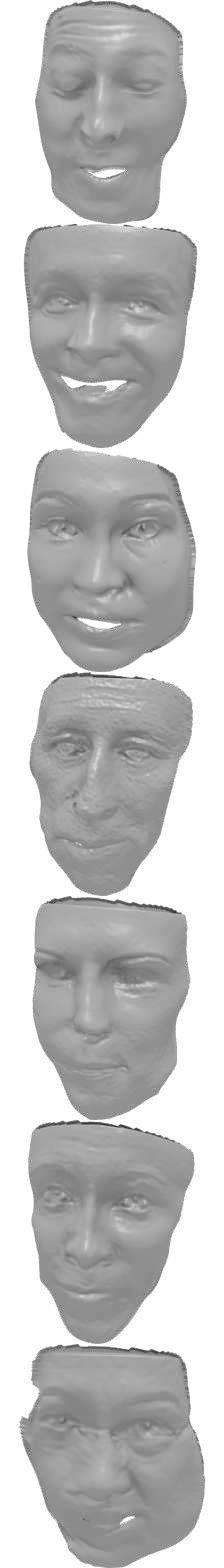}
		\caption{}
	\end{subfigure}
	\quad
	\begin{subfigure}[b]{\mofaspaceour \textwidth}
		\includegraphics[width=\textwidth]{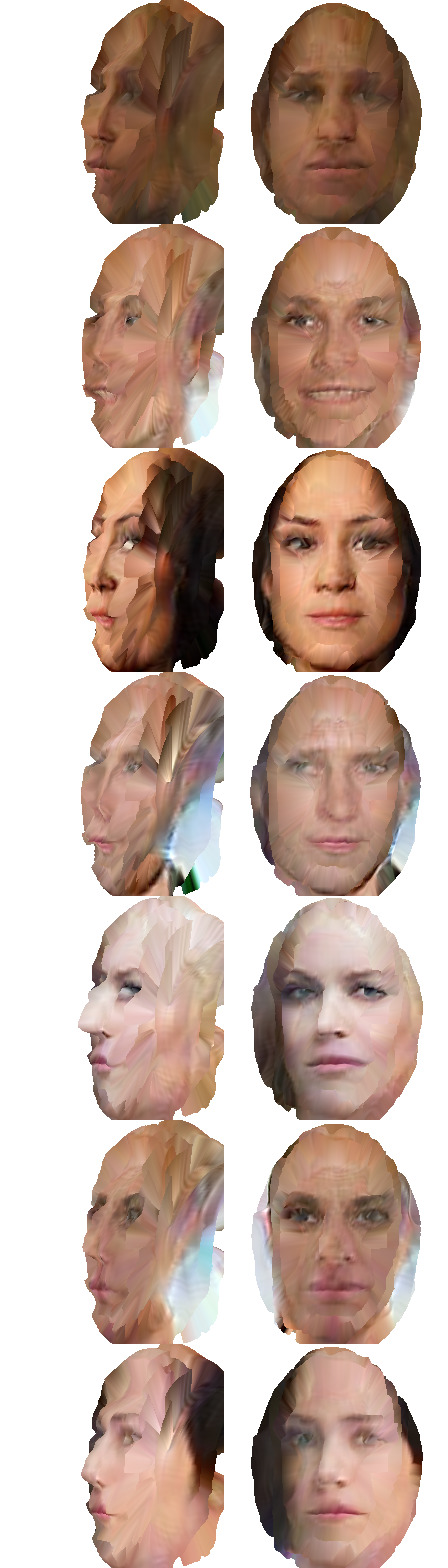}
		\caption{}
	\end{subfigure}
	\quad
	\begin{subfigure}[b]{\mofaspaceour \textwidth}
		\includegraphics[width=\textwidth]{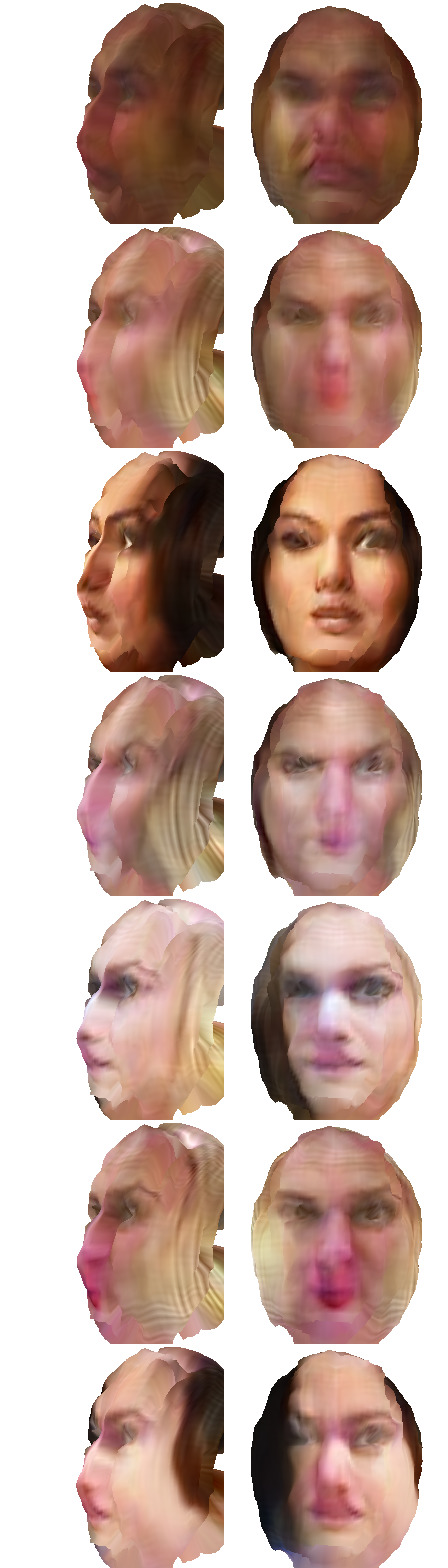}
		\caption{}
	\end{subfigure}
\end{center}
	\caption{Comparisons with other methods on face reconstruction: \textbf{(a)} Input images, \textbf{(b)} Genova \cite{genova2018}, \textbf{(c)} Tran \cite{tran2017}, \textbf{(d)} MoFA \cite{mofa2017},  \textbf{(e)} MoFA + Expressions \cite{mofa2017},  \textbf{(f)} Sela \cite{sela}, \textbf{(g)} ours $\lae$ and \textbf{(h)} ours $\lae+\lsmooth$. Other methods use much stronger supervision on pre-trained models. In contrast our is fully unsupervised.}
\label{fig:mofa}
		\vspace{-3 mm}
\end{figure*}

\section{Experiments}

We trained the generator $\G$ on CelebA \cite{celeba} and ShapeNet \cite{shapenet} datasets and our autoencoder ($\E+\G$) on CelebA. We studied the effect of a regularization term on the surface normals, we shown observations of the hollow-mask and reference ambiguities and compared our method to other face reconstruction methods.

\noindent\textbf{CelebA.}
CelebA has $200$k colored photos of faces, $10\%$ of which are validation and $10\%$ test images. We used the images at a $128\times 128$ pixel resolution. We randomly sampled the Euler angles uniformly in the range of $\pm 15$ degrees for rotating around the horizontal axis and $\pm 65$ degrees around the vertical axis. We did not rotate along the camera axis. For the autoencoder training we increased the bounds to $\pm 75$ and $\pm 20$ degrees respectively and we also allowed rotations along the camera axis by $\pm 15 $ degrees.

Figure~\ref{fig:gan} shows samples from our generator trained on CelebA. For better viewing we rendered them at $256$ pixel resolution. We can see that we achieve plausible textures and 3D shapes. We can clearly see the reconstruction of the nose, brow ridge and the lips. Some smaller details of the 3D are not precise, we can observe some high frequency artifacts and the side of the face has errors as well. However our results are promising, given that it is the first attempt at generating colored 3D meshes on a real dataset without using any annotations.

\noindent\textbf{Smoothing.}
We added a smoothing term in the objective that is meant to make the generated meshes smoother. It is defined as
\begin{align}
	\lsmooth(\G) = E_{\z_f,\v_f} \left[ \sum_{i, j \in \mathcal{N}} 1-\n_i \cdot \n_j \right],
\end{align}
where $i$ and $j$ are indices of neighboring triangles and $\n_i$ and $\n_j$ are the normal vectors of those triangles. During training we optimize $\lgan(\G,\D) + \lambda_S \lsmooth(\G)$ instead of the GAN objective~\eqref{eq:gan}.
We show results with different amounts of smoothing in Figure~\ref{fig:smooth}. We can see that without smoothing the 3D has high frequency artifacts. When $\lambda_S$ is too high, the system cannot learn the correct details of the 3D surface, but only an average ellipsoid. With a moderate amount of smoothing the system reduces the high frequency artifacts, and keeps the larger 3D features.

\noindent\textbf{Ambiguities.}
Figure~\ref{fig:ambiguities} shows the ambiguities discussed in section~\ref{theory}. The hollow-mask ambiguity could be observed on a large proportion of generated samples. Because most faces in the CelebA dataset are close to the frontal view, there are only a few examples that provide self-occlusion cues. However we noticed that the system tried to increase  the size of the object to create better looking hollow-masks. Thus we limited the object size by resizing it when its radius (the maximal radial distance of its vertices) was too large. This helped to eliminate most cases of hollow-masks. We can also see a sample having the reference ambiguity
in Figure~\ref{fig:ambiguities}. The 3D and the texture is plausible, but the reference frame of the object differs from the canonical.

\noindent\textbf{Autoencoder.}
Figure~\ref{fig:autoencoder} shows results with two settings of our autoencoder. The first one is trained by minimizing the auto-encoder loos $\lae$, the other one included the smoothing term too, $\lae + \lambda_S \lsmooth$. The smoothing removes most of the high frequency artifacts of the 3D shape, but tends to generate only an average 3D shape.

\noindent\textbf{Comparisons.}
In Figure~\ref{fig:mofa} we compared our autoencoder to other methods that reconstruct faces from single images. Although the quality of our 3D shapes does not reach the state of the art, we do not use supervision unlike all the other methods.  Tran \etal \cite{tran2017} and Genova \etal \cite{genova2018} regress the parameters of the Basel face model \cite{basel2017}, while Sela \etal \cite{sela} uses synthetic data and MoFA \cite{mofa2017} utilises face scans.

\noindent\textbf{ShapeNet}
ShapeNet consists of 3D models of $55$ object categories, on average $1$k models for each category. We used renderings of the car category from $24$ distinct viewpoint, uniformly spaced around the objects with an elevation of $15$ degrees. We rendered them at a $128x128$ pixel resolution. We made a several changes to the system, so it could learn form the ShapeNet data. We changed the order of rotations along the horizontal and vertical axes, so we could render the mesh in a full circle around the object. We set the background to a constant white, the same colour as the rendered cars had for background. We did not use the resizing technique to constrain the objects in a volume, as the hollow-mask ambiguity did not occur. Otherwise we used the same parameters as we used for the CelebA training. Samples from our generator are shown on Figure~\ref{fig:shapenet}.

\section{Conclusions}

We have presented a method to build a generative model capable of learning the 3D surface of objects directly from a collection of images. Our method does not use annotation or prior knowledge about the 3D shapes in the image collection. The key principle that we use is that the generated 3D surface is correct if it can be used to generate other realistic viewpoints. To create new views from the generated 3D and texture we use a differential renderer and train our generator in an adversarial manner against a discriminator.
Our experimental results on the reconstructed 3D and texture from real and synthesis images showed encouraging results.

{\small
\bibliographystyle{ieee}
\bibliography{refs}
}

\newpage
\onecolumn

\begin{center}
\section*{Supplementary Material}
\end{center}
\vspace{2cm}

\begin{figure*}[h]
\begin{center}	
	\begin{subfigure}[b]{0.4 \linewidth}
		\includegraphics[width=1\linewidth]{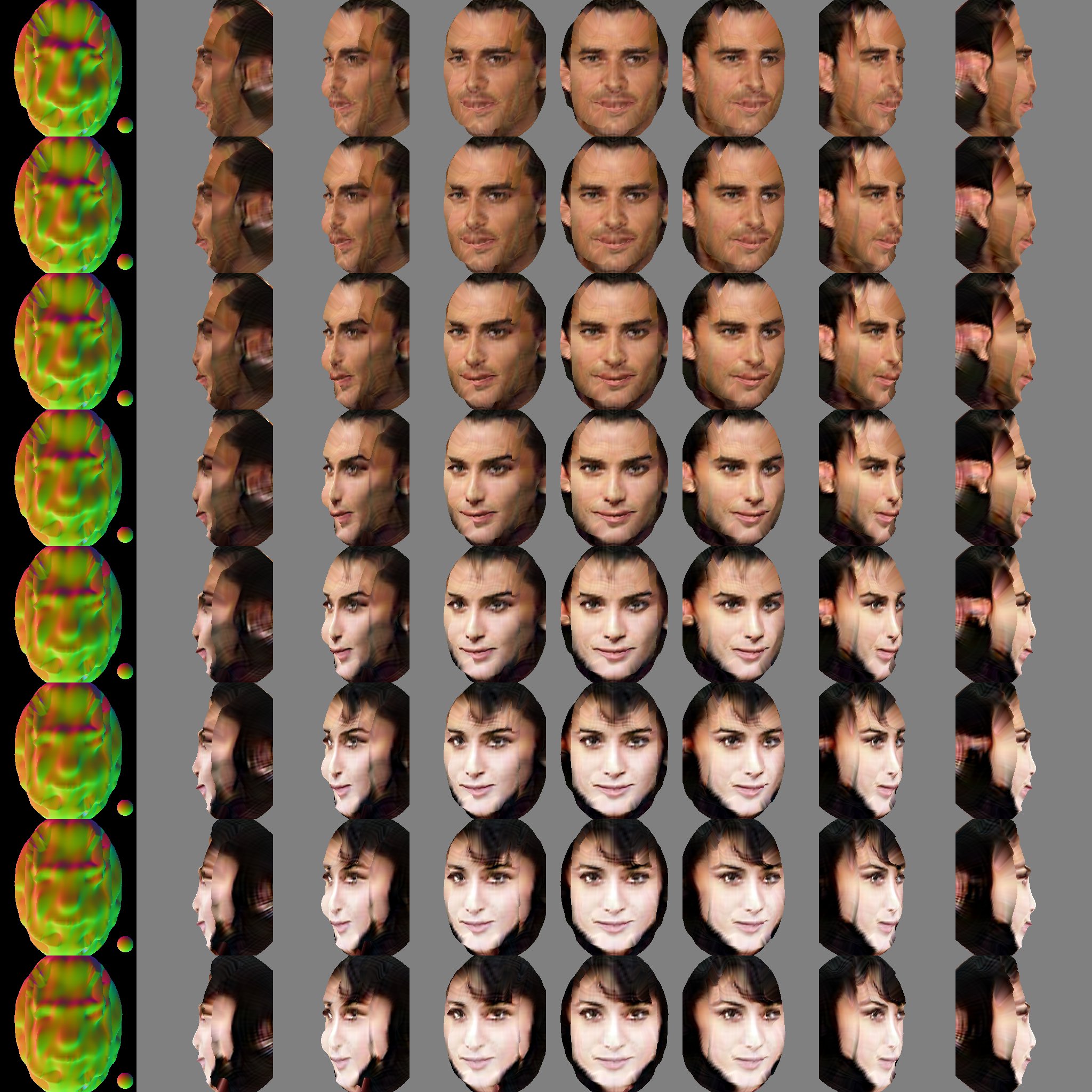}
	\end{subfigure}
	~
	\begin{subfigure}[b]{0.4 \linewidth}
		\includegraphics[width=1\linewidth]{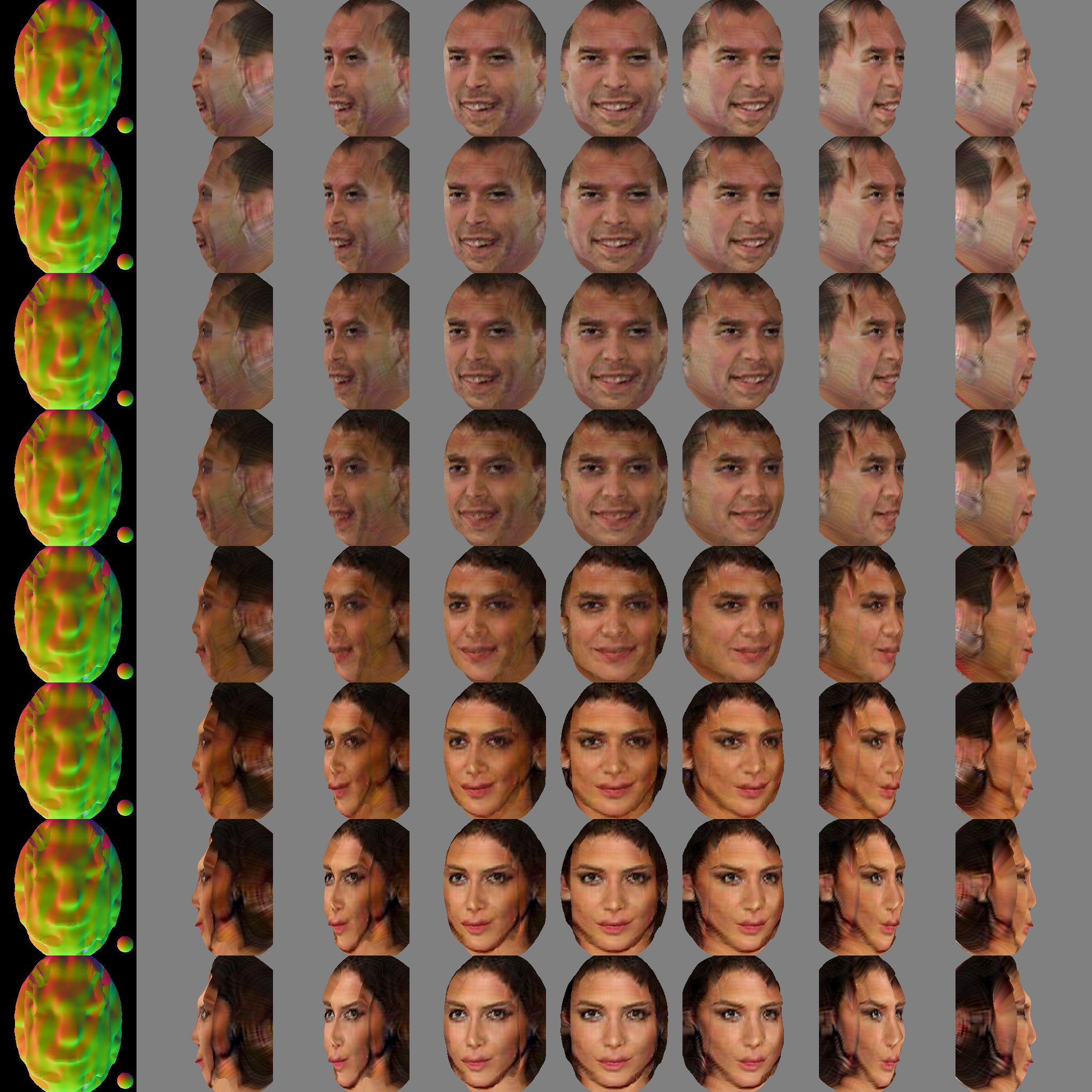}
	\end{subfigure}
	\begin{subfigure}[b]{0.4 \linewidth}
		\includegraphics[width=1\linewidth]{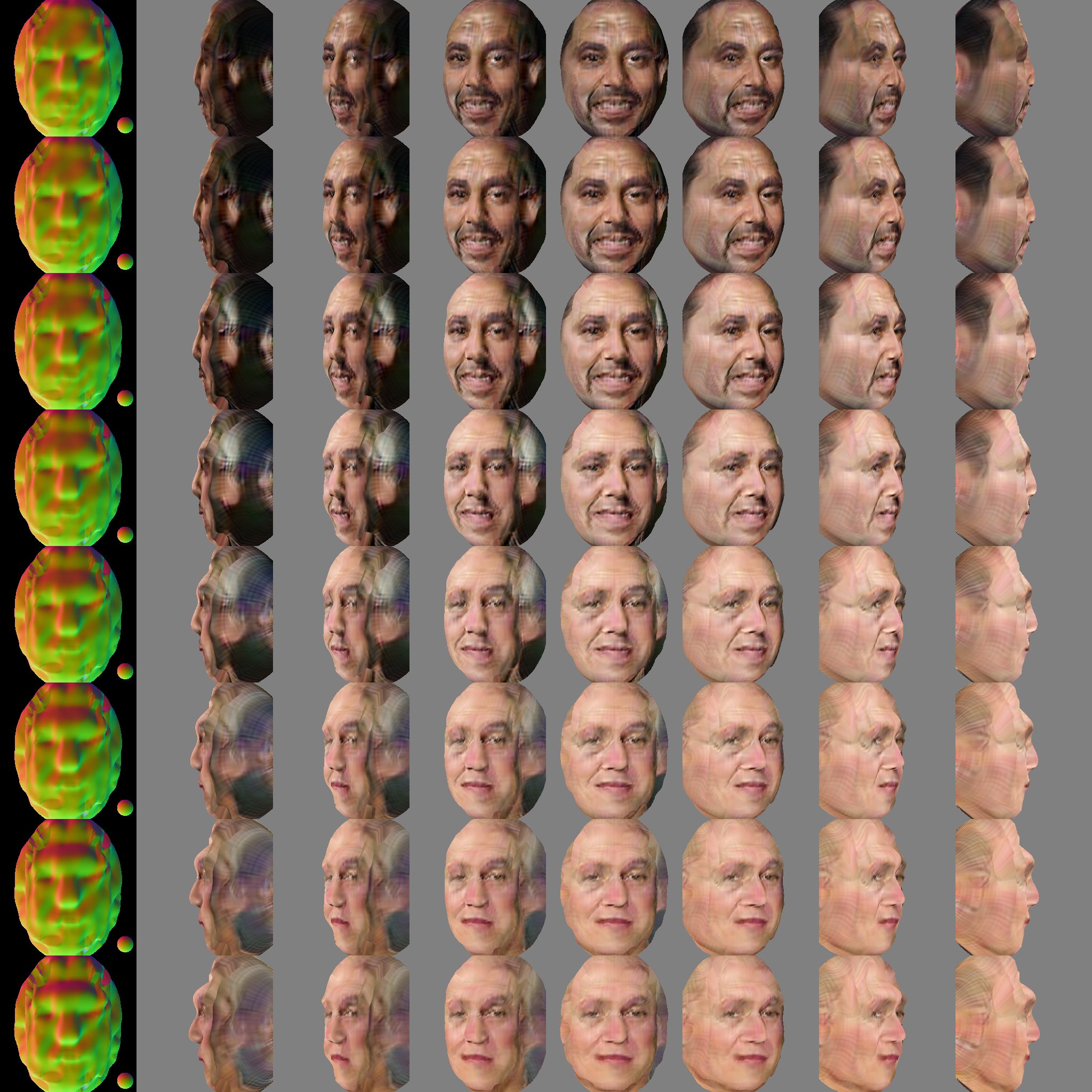}
	\end{subfigure}
	~
	\begin{subfigure}[b]{0.4 \linewidth}
		\includegraphics[width=1\linewidth]{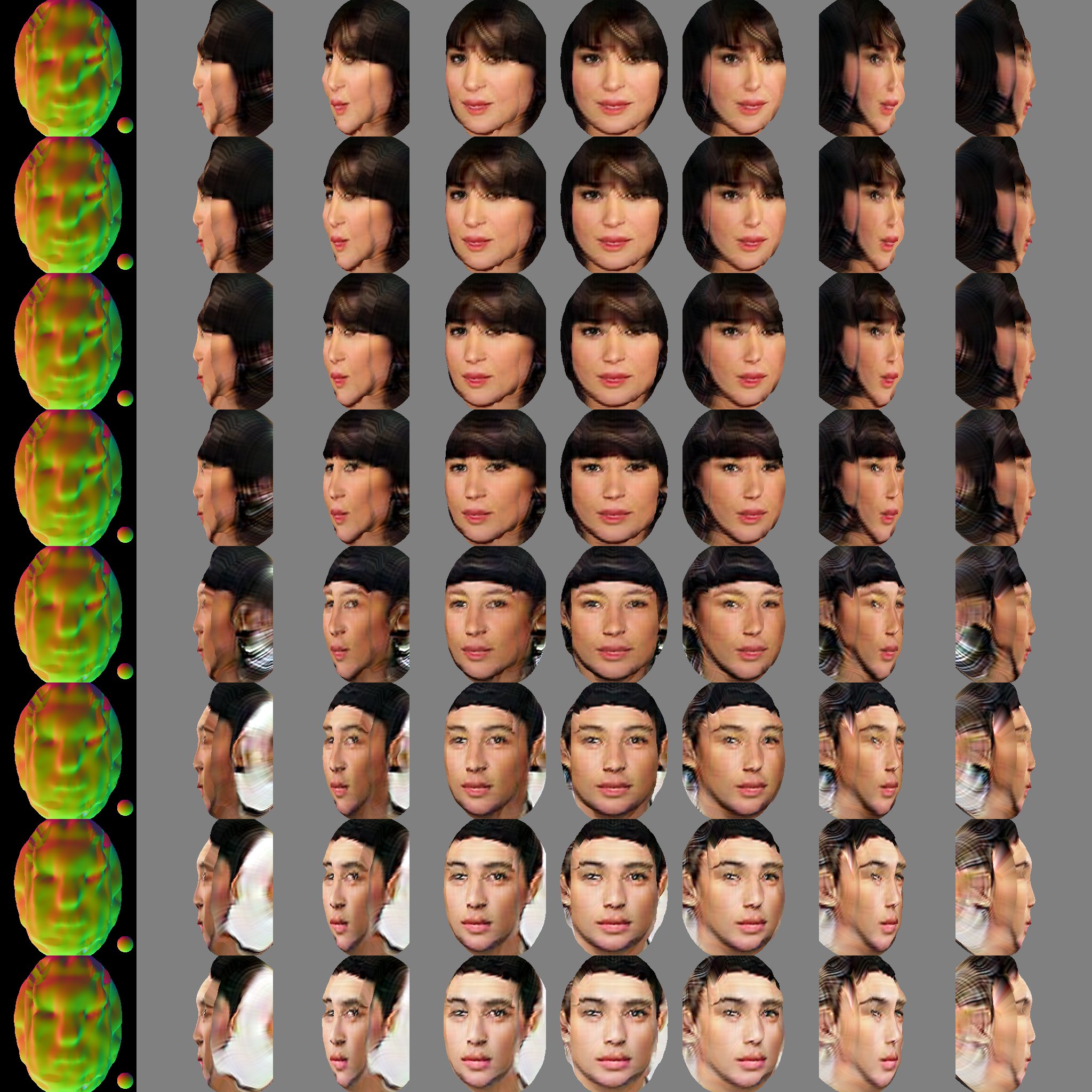}
	\end{subfigure}
	
	\caption{Latent feature interpolations on CelebA. Each row shows renderings of a 3D model $\m = \G(a \z_0 + (1-a) \z_1)$, where $a$ goes from $0$ to $1$. The first column shows the surface normals coloured according to a reference sphere. From the second to the last column we show viewpoints between $\pm 90$ degrees.}
	\label{fig:interpolation}
\end{center}
\end{figure*}

\twocolumn

\begin{figure*}
\begin{center}	
	\begin{subfigure}[b]{0.95 \linewidth}
		\includegraphics[width=1\linewidth]{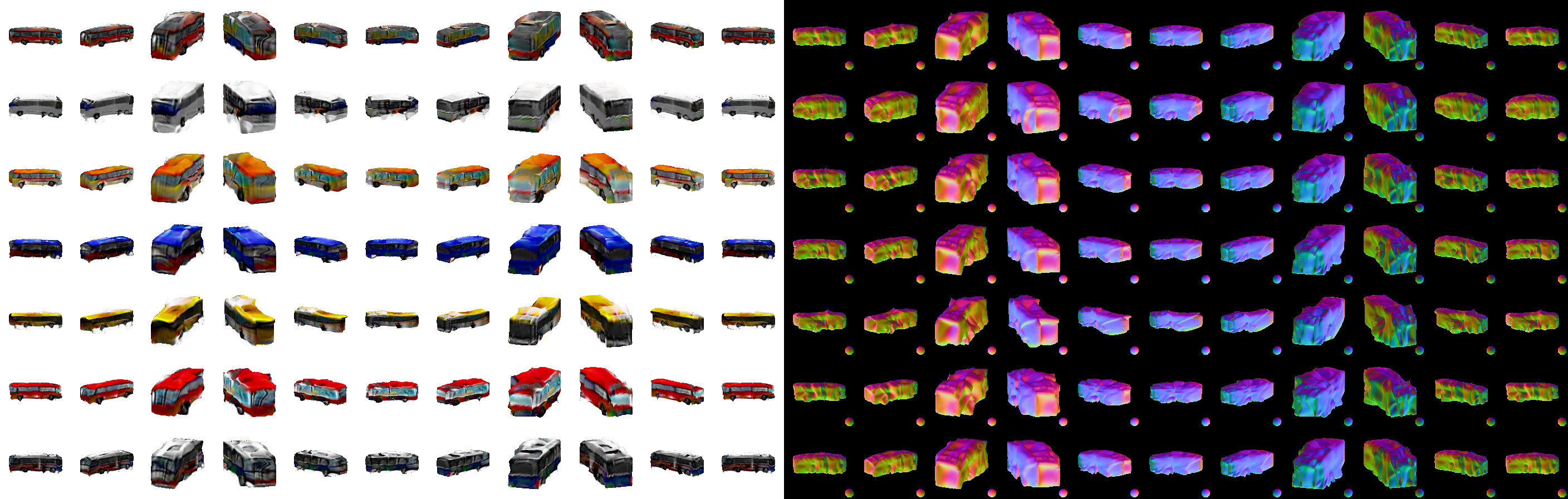}
		\caption{bus}
	\end{subfigure}
	\begin{subfigure}[b]{0.95 \linewidth}
		\includegraphics[width=1\linewidth]{figures/shapenet_gan/car.jpg}
		\caption{car}
	\end{subfigure}
	\begin{subfigure}[b]{0.95 \linewidth}
		\includegraphics[width=1\linewidth]{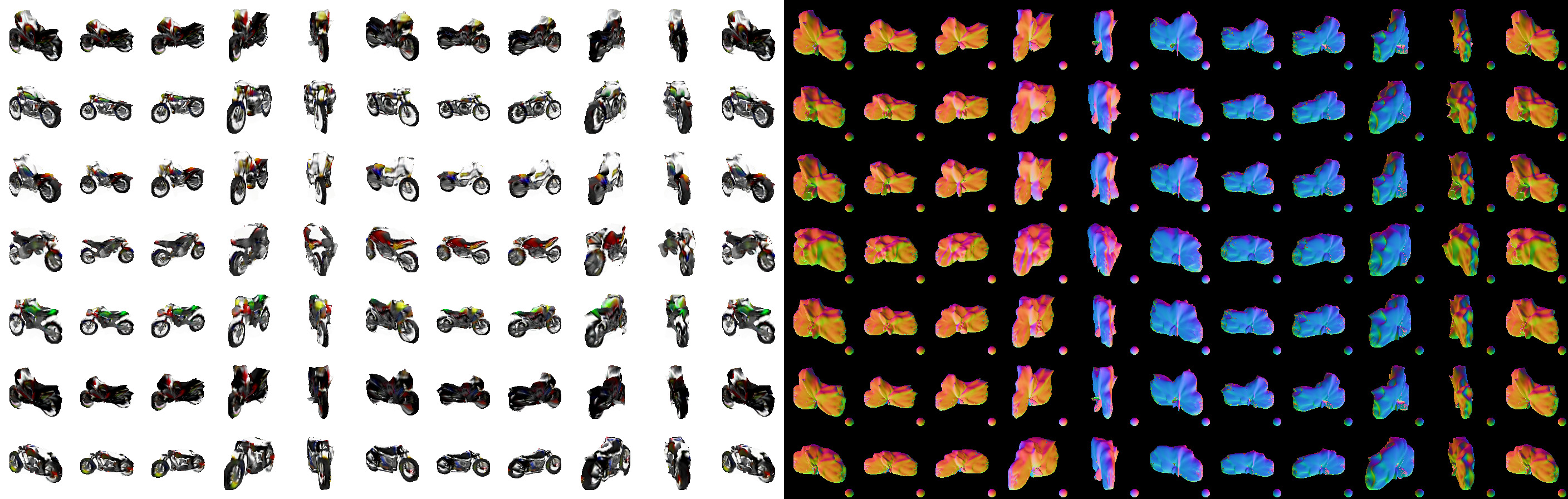}
		\caption{motorcycle}
	\end{subfigure}	
	\caption{Samples from our generator on ShapeNet categories: bus, car, motorcycle. Each row shows a randomly generated ShapeNet object, rendered from viewpoints between $\pm 180$ degrees. The left eleven columns show the rendered images, while the right eleven columns display the surface normals with the reference sphere.}
	\label{fig:shapenet_categories}
\end{center}
\end{figure*}

\end{document}